\documentclass[letterpaper, 10 pt, conference]{ieeeconf}  

\IEEEoverridecommandlockouts                              
\usepackage{graphics} 
\usepackage{epsfig} 
\usepackage{mathptmx} 
\usepackage{times} 
\usepackage{amsmath} 
\usepackage{amssymb}  

\usepackage{mathtools}

\usepackage{algorithm}
\usepackage{algorithmic}
\usepackage{xcolor}
\usepackage{hyperref}

\usepackage{color}

\usepackage{blindtext}
\usepackage{enumerate}
\usepackage{graphicx}
\usepackage{amsfonts, amsmath, bm, amssymb}
\usepackage{dsfont}
\usepackage{wrapfig}
\usepackage{subcaption}

\usepackage{pifont}
\newcommand{\RR}{\mathds{R}}

\usepackage{xspace}
\makeatletter
\DeclareRobustCommand\onedot{\futurelet\@let@token\@onedot}
\def\@onedot{\ifx\@let@token.\else.\null\fi\xspace}

\makeatother

\newcommand{\Lc}{\mathcal{L}}

\newcommand{\Nc}{\mathcal{N}}

\newcommand{\Nb}{\mathbb{N}}

\newcommand{\Rb}{\mathbb{R}}

\newcommand{\mv}{\mathbf{m}}

\newcommand{\sv}{\mathbf{s}}

\newcommand{\wv}{\mathbf{w}}
\newcommand{\xv}{\mathbf{x}}

\ifx\BlackBox\undefined
\newcommand{\BlackBox}{\rule{1.5ex}{1.5ex}}  
\fi
\ifx\QED\undefined
\def\QED{~\rule[-1pt]{5pt}{5pt}\par\medskip}
\fi
\ifx\proof\undefined
\newenvironment{proof}{\par\noindent{\em Proof:\ }}{\hfill\BlackBox\\}
\fi
\ifx\theorem\undefined
\newtheorem{theorem}{Theorem}
\fi
\ifx\example\undefined
\newtheorem{example}{Example}
\fi
\ifx\property\undefined
\newtheorem{property}{Property}
\fi
\ifx\lemma\undefined
\newtheorem{lemma}{Lemma}
\fi
\ifx\proposition\undefined
\newtheorem{proposition}{Proposition}
\fi
\ifx\fact\undefined
\newtheorem{fact}{Fact}
\fi
\ifx\remark\undefined
\newtheorem{remark}{Remark}
\fi
\ifx\corollary\undefined
\newtheorem{corollary}{Corollary}
\fi
\ifx\definition\undefined
\newtheorem{definition}{Definition}
\fi
\ifx\conjecture\undefined
\newtheorem{conjecture}{Conjecture}
\fi
\ifx\axiom\undefined
\newtheorem{axiom}[theorem]{Axiom}
\fi
\ifx\claim\undefined
\newtheorem{claim}[theorem]{Claim}
\fi
\ifx\assumption\undefined
\newtheorem{assumption}{Assumption}
\fi
\ifx\question\undefined
\newtheorem{question}{Question}
\fi
\ifx\problem\undefined
\newtheorem{problem}{Problem}
\fi

\title{\LARGE \bf
Score-based Outlier Generation\\ via Controlling the Radon-Nikodym Derivative
}

\author{Amartya Mukherjee, Tristan Milne, Kry Yik-Chau Lui, Stephanie Hazlewood, Jun Liu
\thanks{This research was supported by Mitacs and the Royal Bank of Canada under the Mitacs Accelerate program award, Application Reference IT49478.}
\thanks{
Amartya Mukherjee and Jun Liu are with the Department of Applied Mathematics, University of Waterloo, Waterloo, Ontario, Canada N2L 3G1 (email: {\tt\small (a29mukhe,j.liu)@uwaterloo.ca}).
}
\thanks{Tristan Milne, Stephanie Hazlewood, and Kry Yik-Chau Lui are with the Royal Bank of Canada, 1 Place Ville Marie, Montreal, Quebec, H3C 3A9 (email: {\tt\small (tristan.milne,stephanie.hazlewood)@rbc.com, yikchau.y.lui@borealisai.com}).}
}

\begin{document}



\maketitle

\begin{abstract}
Outliers are important for stress-testing algorithms and understanding system behaviour under rare conditions.
Despite being commonly described as low-likelihood events, existing generative approaches rarely control likelihood explicitly.
In this work, we introduce a measure-theoretic notion of outliers based on the distribution of log-likelihood values, which is guaranteed to assign higher probability mass to low-likelihood events with a specifiable magnitude.
Building on this formulation, we derive how likelihood reweighting modifies the diffusion score and use this relation to motivate a controlled modification of the reverse-time dynamics.
In particular, likelihood reweighting implies a scaling of the score function with a control term derived from the Radon-Nikodym derivative of the likelihood distributions. Correspondingly, the updated score function can be obtained with no retraining of the diffusion model.
We exploit the Ornstein–Uhlenbeck semigroup underlying diffusion models to motivate an exponentially interpolated controller which approximates the true control.
Experiments demonstrate controlled generation of low-likelihood samples while remaining consistent with the data geometry.
\end{abstract}

Outliers play an important role in evaluating the reliability of algorithms and decision-making systems. In many applications, it is important to understand how systems behave under rare or atypical conditions that deviate from the patterns commonly observed in training or historical data. For example, in time series such as financial data \cite{franses2000non}, outliers model how capital markets respond to anomalous shifts in the financial landscape. In tabular data such as electronic health records \cite{hu2024synthetic}, outliers are crucial for studying patients who deviate from typical disease profiles and for developing treatment strategies tailored to these atypical cases. Generating such rare scenarios is therefore essential for stress-testing algorithms, improving model robustness, and understanding system behaviour under distributional shifts.

Existing approaches to outlier generation suffer from two important limitations. First, although outliers are commonly motivated as low-likelihood samples under a reference distribution \cite{tao2023non,du2023dream}, most methods rely on heuristic notions such as reconstruction error, latent-space distance, or classifier uncertainty rather than explicitly controlling likelihood itself. Hence, they provide little formal interpretability over the statistical rarity of the generated samples. Second, many existing approaches require specialized architectures or training objectives designed specifically for outlier synthesis \cite{gu2025beyond}. Outlier generation is therefore treated as a separate learning problem rather than a controllable sampling problem.

Diffusion models (DMs) have emerged as a powerful framework in generative modelling, achieving remarkable success in domains such as image synthesis~\cite{HoJA2020,rombach2022high,sauer2024fast,peebles2023scalable} and video generation~\cite{ho2022imagen,blattmann2023stable}.
These models are typically trained to reverse a stochastic process that gradually transforms samples from a data distribution into noise via a learned score function. 
These properties also make DMs particularly attractive for outlier generation. DMs admit tractable likelihood estimation through the probability-flow ordinary differential equation (PF-ODE) and evolve according to continuous-time dynamics that can be controlled. In fact, as we will show, a DM trained on regular data can be steered toward low-likelihood regions at inference time, without retraining or fine-tuning.
In this work, we propose a distributional control-theoretic framework for outlier generation in DMs. Unlike previous approaches that identify individual anomalous samples, we control the likelihood distribution. 
The log-likelihood map induces a one-dimensional pushforward measure, and generating outliers corresponds to steering this measure towards lower values.
Our key observation is that likelihood reweighting implies a scaling of the score function with a controller determined by a Radon–Nikodym (RN) derivative on likelihood spaces. This exact pointwise identity motivates a controlled PF-ODE.
To obtain a practical controller, we exploit the exponential convergence of the Ornstein–Uhlenbeck (OU) semigroup that underlies the forward diffusion process. We use an exponentially decaying approximation to the initial likelihood-reweighting control.

Numerical experiments on Gaussian mixture examples and the CIFAR-10 dataset demonstrate that the proposed method successfully steers diffusion sampling toward low-likelihood regions while remaining consistent with the underlying data geometry. 
The results suggest that DMs can be interpreted as controllable systems and opens the door to distribution steering and generative modelling under functional constraints.









\section{Background}

\subsection{Score-Based Diffusion Models}

The diffusion process \cite{SongDKKEP2021} can be parameterized in continuous time by the following Ornstein-Uhlenbeck (OU) SDE
\begin{equation}\label{eq:forward_SDE}
    d\xv_t=-\xv_tdt+\sqrt{2}d\wv_t,\quad\xv_0\sim p_{data}
\end{equation}
for $t\in[0,T]$, where $\xv_0$ is sampled from a data distribution and $\wv_t$ is Brownian motion. As $t$ grows, $\xv_t$ diffuses from the clean data $\xv_0$ into Gaussian noise.
We then generate realistic data by sampling Gaussian noise $\xv_T$ and solving the reverse-time SDE
\begin{equation}\label{eq:reverse_SDE}
    d\xv_t=[-\xv_t-2\nabla\log p_t(\xv_t)]dt+\sqrt{2}d\overline{\wv}_t,
\end{equation}
where $\overline{\wv}_t$ denotes a reverse Brownian motion. Alternatively, we can solve the probability flow ODE (PF-ODE~\cite{SongDKKEP2021}) in reverse-time
\begin{equation}\label{eq:pf_ode_p}
    \dot\xv_t=-\xv_t-\nabla\log p_t(\xv_t),\quad\xv_T\sim\Nc(0,I).
\end{equation}
It is common practice in the DM literature to train a neural network $\sv_\theta(\xv_t,t)$ to approximate $\nabla\log p_t(\xv_t)$, from which realistic training data can then be generated.

\subsection{Fokker-Planck-Kolmogorov Equation}

The Fokker-Planck-Kolmogorov (FPK) \cite{Bogachev2015} equation that governs the evolution of the underlying distribution from the SDEs in Equation \ref{eq:forward_SDE} (forward time) and Equation \ref{eq:reverse_SDE} (reverse time) is given by the partial differential equation (PDE):
\begin{align}
    \frac{dp_t(\xv_t)}{dt}=\nabla\cdot[\xv_tp_t(\xv_t)]+\Delta p_t(\xv_t),
\end{align}
where $\nabla\cdot$ is the divergence and $\Delta$ is the Laplacian operator.
Recently, density steering via controlling the FPK equation has been of interest to the control community \cite{fleig2017optimal,chertovskih2024optimal,sinigaglia2022density}.
It is often convenient to express this evolution in terms of the log-density to obtain log-likelihoods. The following result gives the corresponding PDE satisfied by the log-density.
\begin{proposition}[Proposition 3.1 of \cite{LaiTMUME2023}]\label{prop:FPK}
    Assume the ground truth density $p_t(\xv)$ is sufficiently smooth on $\Rb^n\times[0,T]$ with its log-density denoted as $l_t(\xv):=\log p_t(\xv)$. Then for all $(\xv,t)$, its log-density satisfies the PDE
\begin{equation}\label{eq:log_FPK}
    \partial_t l_t(\xv)
    =\xv\cdot\nabla l_t(\xv)+n+\Delta l_t(\xv)+\|\nabla l_t(\xv)\|^2.
\end{equation}
\end{proposition}

\subsection{Ornstein-Uhlenbeck Operator}

The forward diffusion process underlying score-based models corresponds to OU dynamics, whose generator plays an important role in our analysis of likelihood evolution.

\begin{definition}[Ornstein-Uhlenbeck (OU) Generator]\label{def:ou_generator}
    The OU generator $\Lc$ acts on smooth functions $f\in C^2(\Rb^n)$ by
    \[\Lc f(\xv)=\Delta f(\xv)-\xv\cdot\nabla f(\xv).\]
\end{definition}

\begin{theorem}[Theorem 3.8 of \cite{bogachev2018ornstein}]\label{def:ou_spectrum}
    Let $\gamma:=\Nc(0,I)$ be the standard Gaussian measure. In the Hilbert space $L^2(\gamma),$ $\Lc$ is self-adjoint and has discrete spectrum $\{0,-1,-2,...\}$ of non-positive integers.
    The eigenfunctions are the Hermite polynomials \cite{hermite1864nouveau} $(H_\alpha)_{\alpha\in\Nb^n}$, orthogonal in $L^2(\gamma),$ satisfying $\Lc H_\alpha=-|\alpha|H_\alpha$.
\end{theorem}

\section{Controlled Likelihood Generation}\label{sec:controlled_likelihood_generation}

Outliers are commonly described as samples with low likelihood under a reference data distribution. In likelihood-based generative modelling, this intuition translates into identifying regions where the log-likelihood $l(\xv) := \log p(\xv)$ is small relative to typical data.
However, DMs learn \emph{probability measures}, not individual samples. Thus, if we wish to generate outliers in a principled way, we must define outliers at the level of distributions rather than individual points.

In this section, we characterize outliers through the distribution of log-likelihood values induced by a probability measure. Since outliers correspond to samples with unusually low likelihood, our goal is to steer the likelihood distribution toward lower values. This motivates the use of stochastic ordering to formalize the notion that the target likelihood distribution places more probability mass on lower-likelihood regions. Our main theoretical result shows that likelihood reweighting induces a suitable control through a simple RN structure, resulting in a multiplicative modification of the score function. Motivated by this relation, we use the modified score in a controlled PF-ODE and empirically evaluate its ability to generate low-likelihood samples.

\subsection{Problem Formulation and Definition of Outlier}

Given a data distribution $\mu$ on $\Rb^n$ with density $p$, the log-likelihood function $l(\xv)=\log p(\xv)$ induces a scalar random variable $l(\xv)$ when $\xv\sim\mu$. The pushforward measure $L:=l_\#\mu$ therefore describes the distribution of likelihood values under the model. Shifting this distribution toward lower values corresponds to generating samples that are globally less likely. By modifying this one-dimensional likelihood distribution while preserving the conditional structure of the data given its likelihood level, we obtain a mechanism for generating structured outliers that remain consistent with the underlying data geometry. We now formalize these notions.


\begin{definition}[Log-likelihood pushforward measure]\label{def:likelihood_pushforward}
Let $\mu$ be a probability measure on $\Rb^n$ with density $p$. Define the log-likelihood function
$l(\xv) := \log p(\xv).$
The pushforward measure of $\mu$ under $l$ is the probability measure
$L := l_{\#}\mu$
on $\Rb$.
\end{definition}

\begin{definition}[Likelihood-reweighted measure]\label{def:lr_distribution}
Let $\mu$ be a probability measure on $\Rb^n$ with density $p$ and log-likelihood $l(\xv)=\log p(\xv)$. Let $L=l_{\#}\mu$.
Let $\eta$ be a probability measure on $\Rb$ such that $\eta \ll L$.
A probability measure $\nu$ on $\Rb^n$ is called a likelihood-reweighted measure of $\mu$ with target $\eta$ if
$\nu$ admits the disintegration
\[\nu(A)=\int_{\Rb} \mu(A \mid l(\xv)=u)\,\eta(du),
\quad A \subset \Rb^n \text{ Borel},\]
where $\mu(\cdot \mid l(\xv)=u)$ denotes a regular conditional probability of $\mu$ given $l(\xv)=u$, which exists for Borel probability measures on Polish spaces \cite{leao2004regular}.
\end{definition}

As a consequence, if $\nu$ is a likelihood-reweighted measure of $\mu$ with target $\eta$, then $l_{\#}\nu = \eta$.

\begin{definition}[First-order stochastic dominance (FOSD) \cite{bawa1975optimal}]\label{def:FOSD}
Let $\mu$ and $\nu$ be probability measures on $\Rb$. We say that $\nu$ first-order stochastically dominates $\mu$ and write
\[\mu\leq_{st}\nu\] if and only if $\mu([x,\infty))\leq \nu([x,\infty))$ for all $x\in\Rb$.
\end{definition}
Equivalently, if we let $F_\mu$ and $F_\nu$ be cumulative distribution functions (CDFs) of $\mu$ and $\nu$ respectively, then 
\[\mu\leq_{st}\nu\iff F_\mu(u) \ge F_\nu(u)\quad \text{for all } u \in \Rb.\]

Stochastic dominance constraints are studied in stochastic optimization and decision theory as a way of enforcing preference relations between random outcomes \cite{dentcheva2003optimization,dentcheva2004optimality,dentcheva2022risk}.

\begin{definition}[$\rho$-outlier measure]\label{def:rho_outlier}
Let $\mu$ be a probability measure on $\Rb^n$ with log-likelihood pushforward $L=l_{\#}\mu$. 
A likelihood-reweighted measure $\nu$ with target $\eta$ is called a \emph{$\rho$-outlier measure} if
\[
\text{(1) }\eta \le_{st} L,\quad\text{ and \quad (2) } W_1(L,\eta) \geq \rho,
\]
where $\le_{st}$ represents FOSD (see Definition \ref{def:FOSD}) and $W_1(\cdot,\cdot)$ is the Wasserstein-1 distance (see Equation \eqref{eq:W1} below).
\end{definition}

\begin{remark}
The notion of a $\rho$-outlier measure formalizes outliers at the distribution level, which is natural for generative models that approximate probability measures rather than individual samples.

Let $F_L$ and $F_\eta$ denote the CDFs of $L$ and $\eta$, and let $F_L^{-1}$ and $F_\eta^{-1}$ denote their quantile functions.
The condition $\eta \le_{st} L$ is equivalent to
$F_\eta^{-1}(w) \le F_L^{-1}(w)\quad \text{for all } w\in[0,1].$
Moreover, in one dimension the $W_1$ distance admits the quantile representation \cite{villani2021topics}
\begin{equation}\label{eq:W1}
W_1(L,\eta)=\int_0^1 |F_\eta^{-1}(w)-F_L^{-1}(w)|\,dw.
\end{equation}
Hence, under stochastic dominance,
\[W_1(L,\eta)=\int_0^1 \big(F_L^{-1}(w)-F_\eta^{-1}(w)\big)\,dw,\]
so the $\rho$-outlier condition quantifies a strict downward shift of likelihood quantiles, as expected of outliers.
\end{remark}

To generate samples from the likelihood-reweighted measure $\nu$, we seek to characterize its score $\nabla\log q$ in terms of the score $\nabla\log p$ of the reference distribution, where $p$ and $q$ are the densities of $\mu$ and $\nu$ respectively. This allows us to use an existing diffusion score model while modifying its sampling dynamics through a likelihood-dependent correction.


\subsection{Radon-Nikodym Structure of Likelihood Reweighting}

Based on our formulation of likelihood-reweighted distributions (Definition \ref{def:lr_distribution}), we derive some further properties.

\begin{theorem}\label{thm:score-relationship}
Let $\mu$ be a probability measure on $\Rb^n$ with density $p$ and log-likelihood $l(\xv)=\log p(\xv)$.
Let $L=l_{\#}\mu$ be the pushforward measure of $\mu$ under $l$.
Let $\eta$ be a probability measure on $\Rb$ such that $\eta \ll L$, and define
\[
z(u):=\frac{d\eta}{dL}(u).
\]

Let $\nu$ be the likelihood-reweighted measure defined by
\[
\nu(A):=\int_{\Rb}\mu(A \mid l(\xv)=u)\,\eta(du),
\qquad A \subset \Rb^n \text{ Borel}.
\]

Then $\nu \ll \mu$ and
\begin{equation}
    \frac{d\nu}{d\mu}(\xv)=z(l(\xv))=\frac{d\eta}{dL}(l(\xv)).
\end{equation}
\end{theorem}

\begin{proof}
Let $f:\Rb^n \to \Rb$ be bounded and measurable. By definition of $\nu$,
\[
\int f(\xv)\, \nu(d\xv)
= \int_{\Rb} \left[ \int f(\xv)\, p(d\xv \mid l(\xv) = u) \right] \eta(du).
\]
Since $\eta \ll L$ with density $z(u)$, we can write $\eta(du) = z(u)\,L(du)$
and obtain
\[
\int f(\xv)\, \nu(d\xv)
= \int_{\Rb} z(u)\, \int f(\xv)\,\mu(d\xv \mid l(\xv) = u)\, L(du).
\]
By the disintegration theorem \cite{dellacherie1978probabilities},
\[
\int_{\Rb}\int f(\xv)\,\mu(d\xv \mid l(\xv)=u)L(du)=\int f(\xv)\,\mu(d\xv).
\]
Applying the same identity to the measurable function $x\mapsto z(l(\xv))f(\xv)$ gives
\[
\int f(\xv)\,z(l(\xv))\,\mu(d\xv)
=\int_{\Rb}z(u)\int f(\xv)\,\mu(d\xv \mid l(\xv)=u)L(du).
\]
Comparing the two expressions, we conclude
\[\int f(\xv)\,\nu(d\xv)=\int f(\xv)\,z(l(\xv))\,\mu(d\xv).\]
Since this holds for all bounded measurable $f$, it follows that
\[\frac{d\nu}{d\mu}(\xv)=z(l(\xv)).\]
\end{proof}
\begin{corollary}\label{corr:score-adjustmnt}
Let $\mu_t$ be a family of probability measures on $\Rb^n$ with densities $p_t$ and log-densities $l_t=\log p_t$.
Let $L_t=(l_t)_\#\mu_t$.
Fix a family of target measures $(\eta_t)_{t\in[0,T]}$ on $\Rb$ such that $\eta_t\ll L_t$ for each $t$, and define
\[z_t(u):=\frac{d\eta_t}{dL_t}(u).\]
Define $\nu_t$ to be the likelihood-reweighted measure of $\mu_t$ with target $\eta_t$ in the sense of Definition~\ref{def:lr_distribution}, and let $q_t$ denote its density.
Then $\nu_t\ll \mu_t$ and
\begin{equation}
    \frac{d\nu_t}{d\mu_t}(\xv_t)=z_t(l_t(\xv_t))=\frac{d\eta_t}{dL_t}(l_t(\xv_t)).
\end{equation}
\end{corollary}


The proof follows the same argument as Theorem \ref{thm:score-relationship}.

\begin{remark}
The family $\{\eta_t\}$ may be chosen as an auxiliary interpolation satisfying $\eta_T=L_T$.
In this case, $\nu_T=\mu_T$, which approaches the standard Gaussian for sufficiently large $T$.
The score identity below characterizes each likelihood-reweighted density point-wise in time and motivates the controlled sampling ansatz introduced in Section \ref{sec:implementation}.
\end{remark}


\begin{proposition}\label{prop:score_decomposition}
Let $(\mu_t,p_t,l_t, z_t, L_t, \eta_t,\nu_t)_{t\in[0,T]}$ be defined as in Corollary \ref{corr:score-adjustmnt}.
Assume further that $p_t$ is a $C^1$ density on $\Rb^n$ and that $z_t$ is $C^1$, and let
$q_t(\xv)$ be the density of $\nu_t$ with respect to the Lebesgue measure.
Then
\begin{equation}
\nabla \log q_t(\xv)=[1 + c_t(l_t(\xv))]\nabla \log p_t(\xv),
\end{equation}
where
\begin{equation}
c_t(u):=\partial_u \log z_t(u).
\end{equation}
\end{proposition}

The proof is a direct calculation using the chain rule. The proposition shows that $\nabla_\xv [\log z_t(l_t(\xv))]$ lies in the span of $\nabla \log p_t$ everywhere.
We are now ready to summarize the main theoretical results.

\begin{corollary}\label{corr:combined_score}
Let $(\eta_t)_{t\in[0,T]}$ be a family of target likelihood
distributions satisfying $\eta_t \ll L_t$ for all $t$.
Then the corresponding likelihood-reweighted densities
$q_t$ satisfy the score relation
\begin{equation}
\nabla \log q_t(\xv) = \left[1 + \partial_u\left(\log\frac{d\eta_t}{dL_t}\right)(\log p_t(\xv))\right] \nabla \log p_t(\xv).
\end{equation}
Consequently, likelihood reweighting induces a multiplicative modification of the diffusion score function.
Furthermore, we choose a target likelihood distribution $\eta_0$ that satisfies
\[
\eta_0 \le_{st} L_0,
\quad
W_1(L_0,\eta_0) > \rho,
\]
matching our specification of $\rho$-outliers.
\end{corollary}
Corollary~\ref{corr:combined_score} shows that likelihood control does not require learning a new score function. 
Instead, the score of a likelihood-reweighted density can be expressed using the original diffusion score and a scalar likelihood-dependent correction, without learning an independent score function.
This provides a distribution-steering perspective on outlier generation through likelihood reweighting.

\section{Implementation}\label{sec:implementation}
Motivated by the score relation in Corollary \ref{corr:combined_score}, we introduce the controlled reverse-time ODE ansatz
\begin{equation}
    \dot\xv_t=-\xv_t-\underbrace{(1+c_t(l_t(\xv_t)))\nabla\log p_t(\xv_t)}_{=:\nabla\log q_t(\xv_t)},
\end{equation}
where $c_t(u)=\frac{d}{du}\log\left(\frac{d\eta_t}{dL_t}(u)\right)$ is a coefficient that needs to be determined.
To approximate $c_t$, it is therefore necessary to understand how both $\eta_t$ and $L_t$ evolve over time under the diffusion dynamics. 
In particular, the forward OU process governing the diffusion model induces a contraction of density perturbations toward the Gaussian equilibrium, which we analyze next.
\begin{theorem}\label{thm:OU_convergence}
Let $\xv_t$ solve the OU SDE \eqref{eq:forward_SDE}, where $\xv_0\sim p$.
Let $\mu_t$ denote the time-marginal law of $\xv_t$, and let $p_t$ denote its density.
The invariant measure is the standard Gaussian measure $\gamma=\Nc(0,I)$ with density $p_\gamma$. Define the log-densities
\[l_t(\xv):=\log p_t(\xv),\quad l^*(\xv):=\log p_\gamma(\xv).\]
Assume $\mu_t\ll\gamma$ and denote $h_t(\xv)$ as the RN derivative
$h_t(\xv):=\frac{d\mu_t}{d\gamma}(\xv)$.
Assume $h_0 \in H^1(\gamma)$. 
Then
\[\|h_t - 1\|_{L^2(\gamma)}^2
\le e^{-2t}\|h_0 - 1\|^2_{L^2(\gamma)}.\]
Moreover,
\[
\|\nabla h_t\|_{L^2(\gamma)}\le e^{-t}\|\nabla h_0\|_{L^2(\gamma)}.
\]
Finally, suppose there exists $m>0$ such that $h_t(\xv)\ge m$ for $\mu_t$-almost all $\xv$ and $t$.
Then
\[
\|\ell_t-\ell^*\|_{H^1(\gamma)}=\|\log h_t\|_{H^1(\gamma)}
\le\frac{e^{-t}}{m}\|h_0-1\|_{H^1(\gamma)}.
\]
\end{theorem}

\begin{proof}
The density $p_t$ satisfies the FPK equation
\[\partial_t p_t = \nabla\cdot(\xv p_t) + \Delta p_t.\]

Define the density ratio
$h_t(\xv) := \frac{p_t(\xv)}{\gamma(\xv)}.$
Using the identities
\[
\nabla p_\gamma = -\xv p_\gamma,
\quad
\Delta p_\gamma = (\|\xv\|^2-n)p_\gamma,
\]
one verifies that $h_t$ satisfies
\[
\partial_t h_t = \Lc h_t,
\quad
\Lc := \Delta - \xv\cdot\nabla,
\]
where $\Lc$ is the OU generator introduced in Definition~\ref{def:ou_generator}.
Let
$g_t := h_t - 1.$
Since $\int h_t\, d\gamma = 1$, we have $\int g_t\, d\gamma = 0$.
The evolution equation becomes
\begin{equation}\label{eq:OU_evolution}
\partial_t g_t = \Lc g_t,\quad g_t|_{t=0}=g_0.
\end{equation}
By Theorem~\ref{def:ou_spectrum}, the operator $\Lc$ is self-adjoint on $L^2(\gamma)$ with eigenfunctions given by Hermite polynomials \cite{hermite1864nouveau},
$\Lc H_\alpha = -|\alpha| H_\alpha .$
Expanding $g_0$ in the Hermite basis gives
\[
g_0 = \sum_{\alpha\neq 0} g_\alpha H_\alpha,
\quad
g_\alpha = \langle g_0,H_\alpha\rangle_{L^2(\gamma)}.
\]
Let $(P_t)_{t\ge 0}$ be the OU semigroup on $L^2(\gamma)$ generated by $\Lc$. Since $g_t$ satisfies Equation \eqref{eq:OU_evolution}, we identify $g_t$ with the unique semigroup solution $g_t=P_t g_0.$
Since $P_t$ acts diagonally on the Hermite basis, we can expand this solution
\[
g_t = \sum_{\alpha\neq 0} e^{-|\alpha|t} g_\alpha H_\alpha .
\]
Taking $L^2(\gamma)$ norms yields
\[
\|g_t\|_{L^2(\gamma)}^2=\sum_{k=1}^{\infty} e^{-2kt}\sum_{|\alpha|=k} g_\alpha^2 \le e^{-2t}\|g_0\|^2_{L^2(\gamma)},
\]
where the inequality
follows immediately since $e^{-kt}\le e^{-t}$ for $k\ge1$.
By Lemma 1 of \cite{chen2025logsobolev}, we can bound the gradient:
\[
\|\nabla g_t\|^2_{L^2(\gamma)}\le e^{-2t}\|\nabla g_0\|^2_{L^2(\gamma)}.
\]
Because $g_t=h_t-1$, this is equivalent to
\[
\|\nabla h_t\|_{L^2(\gamma)}\le e^{-t}\|\nabla h_0\|_{L^2(\gamma)}.
\]

Now assume $h_t\ge m>0$. Since $\log$ is $1/m$-Lipschitz on $[m,\infty)$,
\[
|\log h_t(\xv)|\le \frac{1}{m}|h_t(\xv)-1|.
\]
Also, by expanding $\nabla \log h_t(\xv)=\frac{\nabla h_t(\xv)}{h_t(\xv)},$ we have
\[
\|\nabla \log h_t(\xv)\|\le \frac{1}{m}\|\nabla h_t(\xv)\|.
\]
Combining the two inequalities,
\[
\|\log h_t\|_{H^1(\gamma)}\le\frac{1}{m}\|h_t-1\|_{H^1(\gamma)}.
\]
Finally, using the $L^2$ and gradient estimates above,
\[
\|h_t-1\|_{H^1(\gamma)}\le e^{-t}\|h_0-1\|_{H^1(\gamma)}.
\]
Since
\[
\log h_t=\ell_t-\ell^*,
\]
the conclusion follows.
\end{proof}

Theorem~\ref{thm:OU_convergence} shows that $L_t=(l_t)_\#\mu_t$ converges exponentially towards the Gaussian likelihood equilibrium $(l^*)_\#\gamma$. Motivated by this, we choose the target family $\{\eta_t\}$ that likewise vanishes toward the same equilibrium. In particular, the slowest nonconstant OU decay rate $e^{-t}$ motivates the exponentially decaying controller introduced below.


\subsection{Approximating a control function}

In this section, we claim that, provided the log-likelihood pushforward measure $L_0$ and a target $\eta_0$, the exponentially interpolated control function
\begin{equation}
    \tilde c_t(u)=e^{-t}c_0(u)=e^{-t}\partial_u\left(\log\frac{d\eta_0}{dL_0}(u)\right),
\end{equation}
provides a natural approximation choice for the controlled reverse dynamics. The true time‑dependent control $c_t(u)$ is computationally intractable. 
Motivated by the slowest decaying nonconstant term, $e^{-t}$, of the OU semigroup, we choose the exponentially interpolated controller.
We prove that this approximation satisfies a reasonable error bound over both small and large $t$, thus providing a natural admissible approximation that respects the boundary conditions and admits explicit error bounds near both endpoints.



\begin{theorem}\label{thm:ct_interpolation}
Let $(\xv_t)_{t\in[0,T]}$ be an OU diffusion \eqref{eq:forward_SDE}, and let $L_t$ and $\eta_t$ be two families of probability measures on $\RR$ which are absolutely continuous with respect to Lebesgue measure, converge exponentially to the Gaussian likelihood equilibrium, and satisfy:
\[
\eta_0\ll L_0,\quad\eta_0\leq_{st}L_0,\quad W_1(\eta_0,L_0)\geq\rho.
\]
Define the likelihood ratio
\[
z_t(u):=\frac{d\eta_t}{dL_t}(u),\quad c_t(u):=\partial_u \log z_t(u).
\]
Assume:
\begin{enumerate}
\item[(A1)] There exists $L>0$ such that for all $s,t\in[0,T]$,
\[
\|c_t-c_s\|_{L^2(L_t)} \le L|t-s|.
\]
\item[(A2)] There exists $K\ge 1$ such that for all $t\in[0,T]$ and all measurable $f$,
\[
K^{-1}\|f\|_{L^2(L_0)} \le \|f\|_{L^2(L_t)} \le K\|f\|_{L^2(L_0)}.
\]
\end{enumerate}
Define the interpolant
$\tilde c_t(u):=e^{-t}c_0(u).$
Then $\tilde c_0=c_0$, and for every $t\in[0,T]$,
\[
\|\tilde c_t-c_t\|_{L^2(L_t)}
\le
K\Big( (1-e^{-t})\|c_0\|_{L^2(L_0)} + Lt\Big).
\]
In particular,
$
\|\tilde c_t-c_t\|_{L^2(L_t)} = O(t),
$
\end{theorem}

\begin{proof}
Fix $t\in[0,T]$. Add and subtract $c_0$:
\[
\tilde c_t-c_t = (e^{-t}c_0-c_0) + (c_0-c_t).
\]
Take $L^2(L_t)$ norms and apply the triangle inequality:
\[
\|\tilde c_t-c_t\|_{L^2(L_t)}
\le \|(1-e^{-t})c_0\|_{L^2(L_t)} + \|c_t-c_0\|_{L^2(L_t)}.
\]
\begin{align*}
\text{By (A2), }&\quad\|(1-e^{-t})c_0\|_{L^2(L_t)} \le K(1-e^{-t})\|c_0\|_{L^2(L_0)}.\\
\text{By (A1), }&\quad\|c_t-c_0\|_{L^2(L_t)} \le Lt.
\end{align*}
Combine the two bounds to obtain
\[
\|\tilde c_t-c_t\|_{L^2(L_t)}
\le
K\Big( (1-e^{-t})\|c_0\|_{L^2(L_0)} + Lt\Big).
\]
Finally, since $1-e^{-t}\le t$, the right-hand side is $O(t)$.
\end{proof}

\begin{theorem}\label{thm:ct_interpolation2} Let $(\xv_t)_{t\in[0,T]},L_t,\eta_t,\gamma,z_t(u),c_t(u)$ be defined as in Theorem \ref{thm:ct_interpolation}, and let $l_t,\mu_t,\nu_t$ be defined as in Corollary \ref{corr:score-adjustmnt}. Moreover, assume $c_0\in L^\infty$ and suppose the assumptions from Theorem \ref{thm:OU_convergence} are satisfied for $\mu_t$ and $\nu_t$.
Then the exponentially-interpolated controller 
$\tilde c_t(u):=e^{-t}c_0(u)$ 
satisfies the boundary condition $\tilde c_0=c_0$, and moreover 
\[
\|(\tilde c_t-c_t)(l_t(\cdot))\,\nabla \log p_t(\cdot)\|_{L^2(\gamma)}
\le
C e^{-t},
\quad t\in[0,T],
\]
for a constant $C>0$.
In particular,
\[
(\tilde c_t-c_t)(l_t(x))\,\nabla \log p_t(x)
\]
converges exponentially fast to zero in $L^2(\gamma)$.
\end{theorem}

\begin{proof} 
$\log z_t$ can be expanded as 
\[
\log z_t(l_t(\xv))=\log\frac{d\nu_t}{d\mu_t}(\xv)=\log\frac{d\nu_t}{d\gamma}(\xv)-\log\frac{d\mu_t}{d\gamma}(\xv).
\]
From Theorem \ref{thm:OU_convergence} and the triangle inequality, we have 
\begin{align*} 
\|\log z_t(l_t(\cdot))\|_{H^1(\gamma)} 
&\leq\left(\|\log\frac{d\nu_t}{d\gamma}\|_{H^1(\gamma)}+\|\log\frac{d\mu_t}{d\gamma}\|_{H^1(\gamma)}\right)\\
&\leq\frac{e^{-t}}{m}\left(\|\frac{d\nu_0}{d\gamma}-1\|_{H^1(\gamma)}+\|\frac{d\mu_0}{d\gamma}-1\|_{H^1(\gamma)}\right). 
\end{align*}
By definition of the $H^1(\gamma)$ norm, 
\begin{align*} 
\|\nabla\log z_t\|_{L^2(\gamma)} 
&\leq\frac{e^{-t}}{m}\left(\|\frac{d\nu_0}{d\gamma}-1\|_{H^1(\gamma)}+\|\frac{d\mu_0}{d\gamma}-1\|_{H^1(\gamma)}\right). 
\end{align*} 
Since $\nabla\log z_t(l_t(\xv_t))$ can be expanded as $\partial_u\log z_t\nabla l_t(\xv_t)$, 
\[ 
\|\nabla\log z_t(l_t(\cdot))\|_{L^2(\gamma)}=\|\partial_u\log z_t\nabla l_t(\xv_t)\|_{L^2(\gamma)},
\]
which decays exponentially.
Finally, as $c_t=\partial_u \log z_t$ and $\tilde c_t=e^{-t}c_0$, we can write 
\[
\tilde c_t - c_t=e^{-t}\partial_u \log z_0 - \partial_u \log z_t.
\]
Using the triangle inequality,
\begin{align*} 
\|(\tilde c_t - c_t)(l_t(\cdot))\nabla l_t\|_{L^2(\gamma)} 
\leq&~e^{-t}\|c_0\|_{L^\infty}\|\nabla l_t\|_{L^2(\gamma)}\\&+\|\nabla\log z_t(l_t)\|_{L^2(\gamma)}.
\end{align*}
Since Theorem \ref{thm:OU_convergence} yields 
\[
\|\nabla l_t - \nabla \log \gamma\|_{L^2(\gamma)}
=\|\nabla l_t + x\|_{L^2(\gamma)}
\le C_1e^{-t},
\]
we obtain that $\|\nabla l_t\|_{L^2(\gamma)}$ is uniformly bounded over $t$.
Thus, 
\begin{align*}
\|(\tilde c_t&-c_t)(l_t(\cdot))\,\nabla \log p_t(\cdot)\|_{L^2(\gamma)}
\le
C e^{-t},\\
\text{where }C
=&~\|c_0\|_{L^\infty(\gamma)}\sup_{t\in[0,T]}\|\nabla l_t\|_{L^2(\gamma)}\\
&+\frac{1}{m}(\|\frac{d\nu_0}{d\gamma}-1\|_{H^1(\gamma)}+\|\frac{d\mu_0}{d\gamma}-1\|_{H^1(\gamma)})
\end{align*}
and therefore the right-hand side is $O(e^{-t})$.
\end{proof}

\begin{corollary}
    Under the assumptions of Theorems \ref{thm:ct_interpolation} and \ref{thm:ct_interpolation2}, the error between the true control term $c_t(u)\nabla l_t$ and the estimate $\tilde c_t(u)\nabla l_t=e^{-t}c_0(u)\nabla l_t$ are bounded for small and large $t$:
    \begin{equation}
        \|(\tilde c_t - c_t)(l_t(\cdot))\nabla l_t\|_{L^2(\gamma)}\leq O(\min(t,e^{-t})).
    \end{equation}
\end{corollary}

\subsection{Approximating the target Radon-Nikodym derivative}

We will rely on chi-square approximations for $L_0$ in this paper. In our numerical experiments, these approximations will be verified and demonstrated. We introduce the notation of a linear map $L_{a,C}$ defined as $L_{a,C}(x)=-ax+C$.
Under this approximation, the control coefficients admit closed-form expressions, making the proposed controller straightforward to compute in practice.

\begin{assumption}\label{as:chi_square}
Let $L_0$ be a likelihood-pushforward measure on $\Rb$ with finite mean. Assume there exists $a>0$ and $C$ such that $(L_{a,C})_{\#}(\chi_n^2)\le_{st} L_0$.
\end{assumption}

\begin{fact}\label{fact:target_exists}
Let $a>0$ and consider the measure $P_a:=(L_{a,C})_{\#}(\chi_n^2)$ on $\Rb$.
Then for any $\rho>0$, we can define a measure $P_b:=(L_{b,C})_{\#}(\chi_n^2)$, where $b = a + \frac{\rho}{n}$. This yields
\begin{align*}
    P_b \le_{st} P_a,\quad W_1(P_a,P_b) = \rho.
\end{align*}
\end{fact}


\begin{fact}\label{fact:rn_derivative}
Let $P_a=(L_{a,C})_{\#}(\chi_n^2)$ and $P_b=(L_{b,C})_{\#}(\chi_n^2)$, with $a,b>0$ and $C\in\Rb$. Their RN derivative is
\[
\frac{dP_b}{dP_a}(u) = \left(\frac{a}{b}\right)^{n/2} \exp\left[ (C-u)\left(\frac{1}{2a} - \frac{1}{2b}\right) \right], \quad u<C.
\]
Correspondingly, we obtain the constant control
\[
c_0 := \frac{d}{du}\log\frac{dP_b}{dP_a}(u) = -\frac{1}{2a} + \frac{1}{2b}.\]
\end{fact}

\begin{remark}\label{rem:logz_evolution}
Under the control $\tilde c_t(u)=c_0 e^{-t}$, where $c_0=-\frac{1}{2a}+\frac{1}{2b}$, the corresponding proxy log-RN derivative takes the approximate form
\[
\log z_t(u) \approx c_0 e^{-t} u + \beta(t)\implies z_t(u)\approx\exp(c_0e^{-t}u+\beta(t)),
\]
where $\beta(t)$ is determined by the normalization condition $\int z_t  dL_t = 1$ and the evolution of the measure $L_t$.
This ensures a smooth interpolation between $z_T\approx 1$ and $z_0$ as given in Fact~\ref{fact:rn_derivative}. The linear dependence on $u$ greatly simplifies the implementation and analysis.
\end{remark}

\section{Numerical Experiments}

\subsection{Mixture of Two Gaussians}

We first evaluate our method on a synthetic testbed consisting of a mixture of two Gaussians
\begin{equation}
\xv \sim \tfrac12 \Nc(\mv,I) + \tfrac12 \Nc(-\mv,I).
\end{equation}
For this model the score function admits the closed-form expression \cite{shah2023learning}
\begin{equation}
\nabla\log p_t(\xv)
=
\tanh(\mv_t\cdot\xv)\mv_t-\xv,
\quad
\mv_t = e^{-t}\mv .
\end{equation}

We choose $\mv=\hat{m}1_n$, where $1_n$ denotes the $n$-dimensional vector of ones.
This closed-form score allows us to solve the PF-ODE \eqref{eq:pf_ode_p} and the log-FPK equation \eqref{eq:log_FPK} directly without training a neural network. Our goal is to guide the reverse diffusion process to generate $\rho$-outliers.

\paragraph{Assessment of Assumption \ref{as:chi_square}.}
We empirically assess the approximation of $L_0$ by a translated $-\tfrac{1}{2}\chi_n^2$ distribution. 
FOSD is assessed by sorting samples from two likelihood distributions and verifying that the respective inequality holds element-wise.
To estimate $L_0$, we sample 10,000 points by solving the reverse PF-ODE \eqref{eq:pf_ode_p}, then use the log-FPK equation \eqref{eq:log_FPK} to compute their log-likelihood. 
Table~\ref{tab:verifying_assumption} reports their empirical $W_1$ distance across several dimensions and values of $\hat{m}$. The results indicate that the chi-square model provides a good approximation of the log-likelihood distribution. An example empirical CDF is shown in Figure~\ref{fig:logp_density_controlled_reverse_sde_sample}.

\begin{table}[ht]
\centering
\begin{tabular}{c|c c c c c}
Dimension $n$ & 1 & 4 & 16 & 64 & 256 \\
\hline
$\hat{m}=2$  & 0.0877 & 0.0252 & 0.0961 & 0.1582 & 0.2725 \\
$\hat{m}=4$  & 0.0140 & 0.0545 & 0.0761 & 0.0853 & 0.2964 \\
$\hat{m}=8$  & 0.0113 & 0.0290 & 0.0484 & 0.1461 & 0.3186 \\
$\hat{m}=16$ & 0.0174 & 0.0545 & 0.0628 & 0.1824 & 0.3250
\end{tabular}
\caption{Empirical assessment of the chi-square approximation with $a=1/2$. Values report the empirical $W_1$ distance between $L_0$ and the fitted chi-square distribution.}
\label{tab:verifying_assumption}
\vspace{-0.2in}
\end{table}

\paragraph{Outlier generation.}
For the remaining experiments we set $\hat{m}=2$ and diffusion horizon $T=80$. We vary the dimension $n\in\{1,4,16,64,256\}$ and target $\rho\in\{0.2,0.3,0.4,0.5\}$. In each experiment, we verify that FOSD is empirically satisfied, and report the $W_1$ distance between the generated likelihood distribution $\eta_0$ and the reference $L_0$.

Our results are posted in Table~\ref{tab:W1_distance}. The generated samples match the target $\rho$ values while maintaining stochastic dominance.
Figure~\ref{fig:reverse_PF-ODE} compares histograms of samples generated by the uncontrolled and controlled PF-ODE in the one-dimensional case. The controlled dynamics produce samples concentrated in lower-likelihood regions such as the tails and the low-density region between the mixture modes.
\begin{table}[ht]
\centering
\begin{tabular}{c|c c c c c}
Dimension $n$ & 1 & 4 & 16 & 64 & 256 \\
\hline
$\rho=0.2$ & 0.2286 & 0.2051 & 0.2234 & 0.2248 & 0.2202 \\
$\rho=0.3$ & 0.3379 & 0.3056 & 0.3314 & 0.3368 & 0.3324 \\
$\rho=0.4$ & 0.4266 & 0.4078 & 0.4450 & 0.4498 & 0.4439 \\
$\rho=0.5$ & 0.5065 & 0.5140 & 0.5475 & 0.5627 & 0.5589
\end{tabular}
\caption{Outlier generation results for the Gaussian mixture model. Values report the empirical $W_1$ distances between the generated and reference likelihood distributions. The measured distances closely match the prescribed target values $\rho$, demonstrating accurate control over outlier magnitude while satisfying stochastic dominance in every experiment.}
\label{tab:W1_distance}
\vspace{-0.3in}
\end{table}
\begin{figure}[ht]
\centering
\includegraphics[width=0.9\linewidth,trim={0 0 0 0.3in},clip]{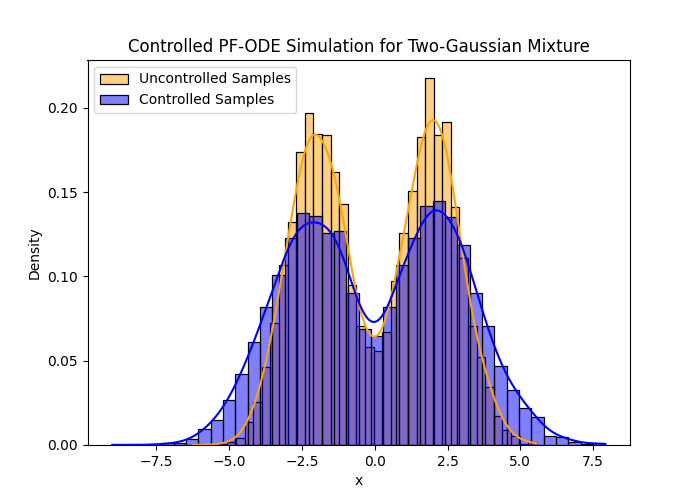}
\vspace{-0.1in}
\caption{Histogram comparison of PF-ODE samples in the 1D Gaussian mixture example with $\rho=0.5$.}
\label{fig:reverse_PF-ODE}
\includegraphics[width=0.9\linewidth,trim={0 0 0 0.3in},clip]{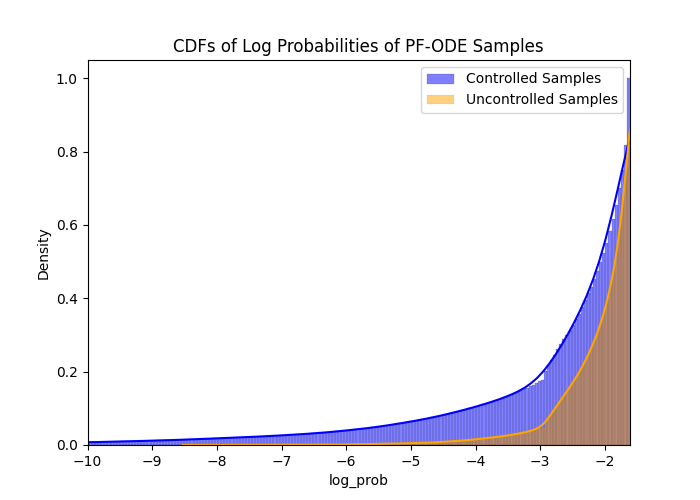}
\vspace{-0.1in}
\caption{Empirical CDFs of the likelihood distribution $L_0$ and the target distribution $\eta_0$ in the 4D example with $\rho=0.5$.}
\label{fig:logp_density_controlled_reverse_sde_sample}
\vspace{-0.2in}
\end{figure}

\subsection{Image Data: CIFAR-10}

We next evaluate the method on the CIFAR-10 dataset using the pretrained diffusion model \texttt{google/ddpm-cifar10-32} \cite{HoJA2020}. The likelihood distribution $L_0$ is approximated using samples generated by the reverse PF-ODE together with the log-FPK equation \eqref{eq:log_FPK}. 

Empirically, $L_0$ stochastically dominates a translated $-\tfrac12\chi_n^2$ distribution, supporting Assumption~\ref{as:chi_square}.
Their empirical $W_1$ distance is $9.137$.
We vary the target $\rho$ across $\{50,100,150,200,250\}$ and generate $\rho$-outliers by modifying the reverse diffusion dynamics. Table~\ref{tab:W1_distance_CIFAR} reports the empirical $W_1$ distances between the generated likelihood distribution and the reference distribution.
Although the realized $W_1$ shifts do not exactly match the prescribed $\rho$ values, they increase monotonically with $\rho$. Thus, even when the score is represented by a neural network, the control parameter provides a consistent mechanism for adjusting the degree of likelihood shift.

Figure~\ref{fig:cifar10_cdf} confirms that increasing $\rho$ progressively shifts the generated likelihood distribution toward lower values. At moderate control strengths, the samples in Figures~\ref{fig:sample_grid_rho_50} and~\ref{fig:sample_grid_rho_100} remain visually consistent with CIFAR-10, suggesting that the controller can access lower-likelihood regions without immediately destroying the learned data structure.

\begin{figure}[ht]
\centering
\includegraphics[width=0.9\linewidth]{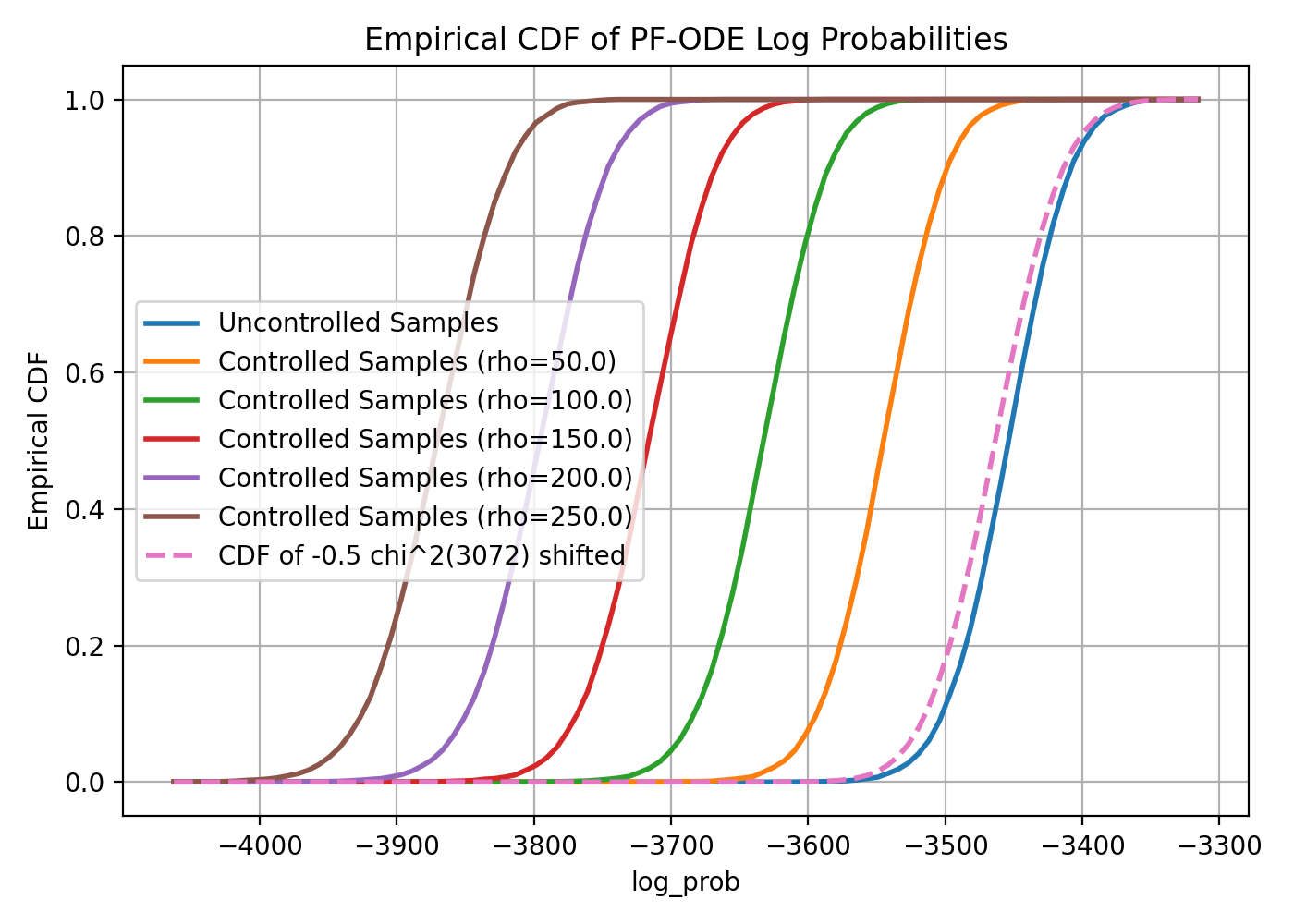}
\vspace{-0.1in}
\caption{Empirical CDFs of likelihood distributions for different values of $\rho$.}
\label{fig:cifar10_cdf}
\vspace{-0.1in}
\end{figure}
\begin{table}[ht]
\centering
\begin{tabular}{c|c c c c c}
$\rho$ & 50 & 100 & 150 & 200 & 250 \\
\hline
$W_1$ distance & 71.187 & 122.658 & 178.799 & 262.563 & 342.121
\end{tabular}
\caption{Outlier generation on CIFAR-10. Increasing the target parameter $\rho$ produces progressively larger $W_1$ shifts in the likelihood distribution, demonstrating that the proposed controller remains effective on image data.}
\label{tab:W1_distance_CIFAR}
\vspace{-0.1in}
\end{table}
\begin{figure*}
    \centering
    \begin{subfigure}[t]{0.45\textwidth}
        \centering
        \includegraphics[width=\linewidth]{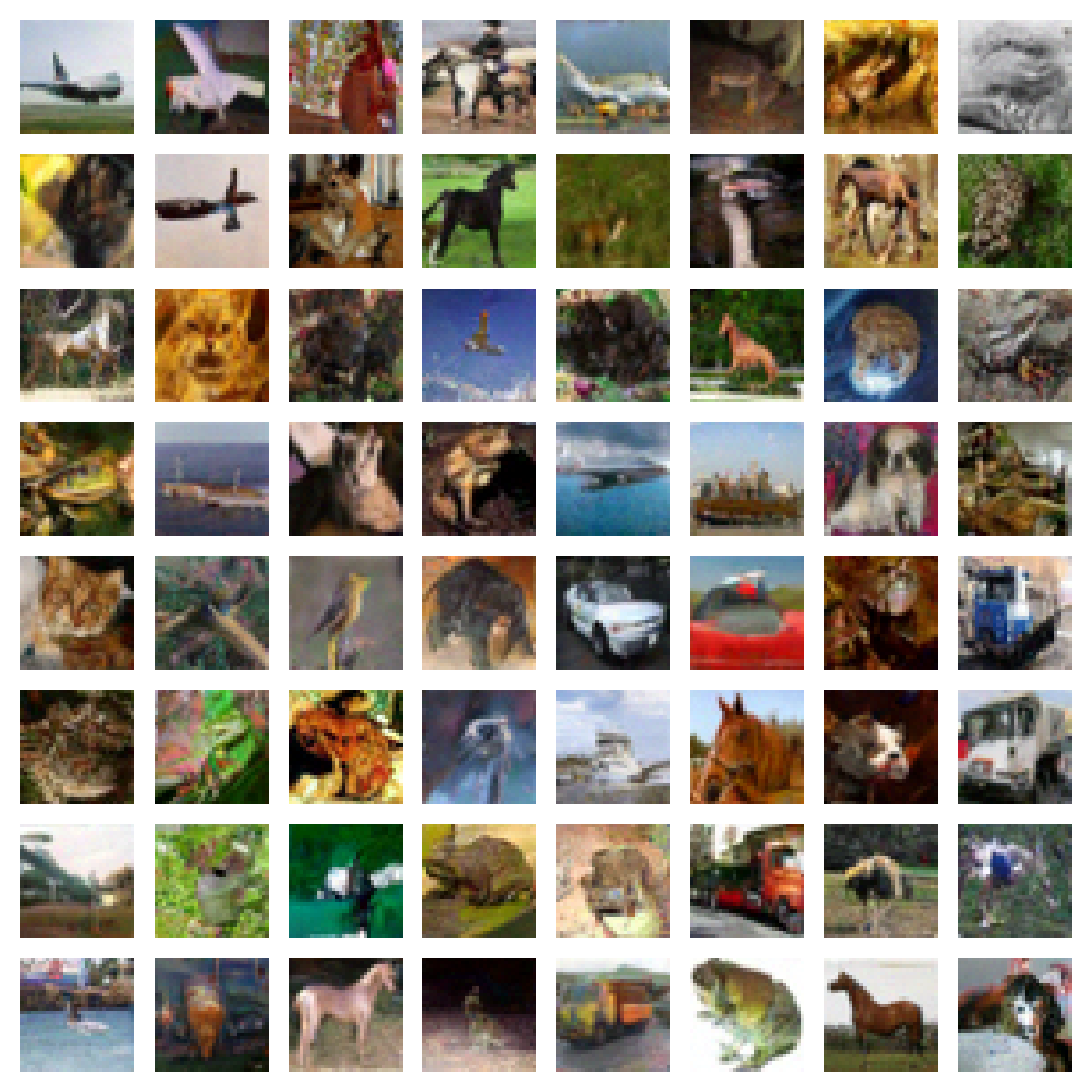}
        \vspace{-0.2in}
        \caption{Generated samples with $\rho=50$.}
        \label{fig:sample_grid_rho_50}
    \end{subfigure}%
    \hspace{0.05\textwidth}
    \begin{subfigure}[t]{0.45\textwidth}
        \centering
        \includegraphics[width=\linewidth]{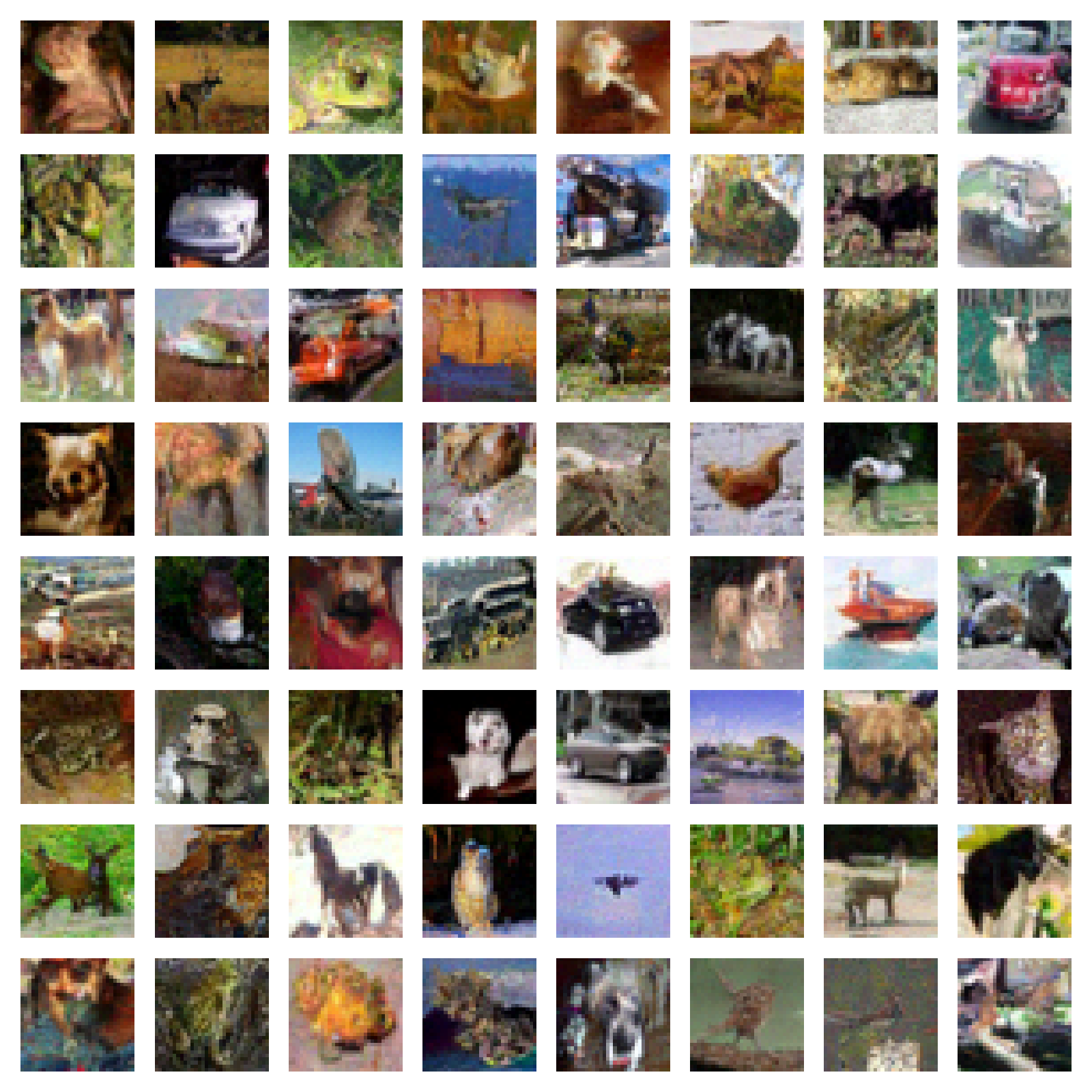}
        \vspace{-0.2in}
        \caption{Generated samples with $\rho=100$.}
        \label{fig:sample_grid_rho_100}
    \end{subfigure}
    \caption{Generated $\rho$-outliers in the CIFAR-10 image dataset. Although the likelihood values are lower, as demonstrated in Figure \ref{fig:cifar10_cdf}, the generated samples remain visually consistent with the underlying data distribution.}
    \label{fig:CIFAR}
    \vspace{-0.2in}
\end{figure*}



Overall, these experiments demonstrate that the proposed framework can steer diffusion sampling toward prescribed likelihood statistics. In both synthetic and image datasets, the controlled PF-ODE successfully shifts the likelihood distribution toward lower-probability regions while preserving the structure of the data distribution.

\section{Conclusion}

In this work, we introduced a distributional perspective on outliers motivated by the probabilistic structure of diffusion models. Instead of defining outliers as individual low-likelihood samples, we formalized them through the distribution of log-likelihood values and proposed a notion of $\rho$-outlier based on the discrepancy between likelihood distributions. Using this formulation, we developed a method for generating outliers by modifying the reverse-time diffusion dynamics. 
The key insight is that likelihood reweighting induces a simple modification of the score function via the RN derivative, which motivates the controlled PF-ODE used for sampling. 
Numerical experiments support the theoretical predictions on synthetic and image datasets. 
Future work will explore broader applications of this framework to distribution steering.

\bibliography{main}
\bibliographystyle{abbrv}

\end{document}